\documentclass{article}

\usepackage{preprint}
\usepackage[T1]{fontenc}
\usepackage{wrapfig} 
\usepackage{enumitem}
\usepackage{amsmath,amssymb,amsthm}
\usepackage{multirow}
\usepackage{algorithm}
\usepackage{algpseudocode}
\usepackage{graphicx}
\usepackage{booktabs}
\usepackage{microtype}

\usepackage{xcolor}
\usepackage{tikz}
\usetikzlibrary{arrows.meta}

\definecolor{burntumber}{rgb}{0.54, 0.2, 0.14}
\definecolor{cobalt}{rgb}{0.0, 0.28, 0.67}

\usepackage{url}
\usepackage{hyperref}

\definecolor{citeblue}{HTML}{4A90D9}

\hypersetup{
    colorlinks=true,
    citecolor=citeblue,
    linkcolor=citeblue,
    urlcolor=citeblue,
    pdftitle={ThinQuant: Fast and Scalable Rotation Learning for Weight and Activation Quantization of Large Language Models},
    pdfauthor={Mehdi Makni, Ryan Lucas, Rahul Mazumder}
}

\newcommand{\R}{\mathbb R}
\newcommand{\SO}{\operatorname{SO}}
\newcommand{\St}{\operatorname{St}}
\newcommand{\conv}{\operatorname{conv}}
\newcommand{\rank}{\operatorname{rank}}
\newcommand{\argmax}{\operatorname*{arg\,max}}

\newcommand{\cA}{\mathcal A}

\newcommand{\cN}{\mathcal N}

\newtheorem{theorem}{Theorem}

\makeatletter
\g@addto@macro\normalsize{%
  \setlength\abovedisplayskip{3pt plus 1pt minus 1pt}%
  \setlength\belowdisplayskip{3pt plus 1pt minus 1pt}%
  \setlength\abovedisplayshortskip{2pt plus 1pt}%
  \setlength\belowdisplayshortskip{3pt plus 1pt minus 1pt}%
}
\makeatother
\setlist[itemize]{topsep=2pt,itemsep=1pt,parsep=0pt,partopsep=0pt}

\usepackage[table]{xcolor}

\definecolor{darkgreen}{HTML}{1B6B45}

\newif\ifdraftnotes
\draftnotesfalse

\definecolor{hull}{HTML}{245D73}
\definecolor{witness}{HTML}{B65324}
\definecolor{direction}{HTML}{624787}
\definecolor{cloud}{HTML}{9BABB2}
\definecolor{ink}{HTML}{23313A}
\definecolor{paperblue}{HTML}{F1F6F8}

\color{ink}

\newcommand{\HullPath}{%
  (3,0)--(2,2)--(0,3)--(-3,0)--(-2,-2)--(0,-3)--cycle%
}

\DeclareMathOperator{\Ext}{ext}

\newcommand{\InteriorCoords}{%
  .25/.35,
  .75/.2,
  1.15/.55,
  1.45/.1,
  1.7/.65,
  .3/1.05,
  .85/1.3,
  1.35/1.35,
  .1/1.85,
  .6/2.1,
  -.55/.35,
  -1.1/.25,
  -.2/1.35,
  -.7/.9,
  .25/-.75,
  .55/-.95,
  1.2/-.45,
  1.85/-.3%
}

\definecolor{ThinBlueTop}{HTML}{A9DDFF}
\definecolor{ThinBlueBottom}{HTML}{0047D7}

\definecolor{ThinEmerald}{HTML}{00A878}
\definecolor{ThinBracket}{HTML}{30343A}

\newlength{\ThinQuantIHeight}

\DeclareRobustCommand{\ThinI}{%
  \begingroup
  \settoheight{\ThinQuantIHeight}{n}%
  \tikz[
    baseline=0pt,
    x=.08em,
    y=1.03\ThinQuantIHeight
  ]{%
    \fill[ThinEmerald]
      (0,0) rectangle (1,1);

    \draw[
      ThinBracket,
      line width=.28pt
    ]
      (-.30,0) -- (-.30,1)
      (-.30,1) -- (-.08,1)
      (-.30,0) -- (-.08,0);

    \draw[
      ThinBracket,
      line width=.28pt
    ]
      (1.30,0) -- (1.30,1)
      (1.08,1) -- (1.30,1)
      (1.08,0) -- (1.30,0);
  }%
  \endgroup
}

\DeclareRobustCommand{\ThinQuantLogo}{%
  {\sffamily\bfseries
    Th%
    \kern.035em%
    \ThinI%
    \kern.035em%
    nQuant%
  }%
}

\title{%
  \texorpdfstring{%
    Th\kern.035em\ThinI\kern.035em nQuant: Scalable Rotation Learning\\
    for Weight and Activation Quantization of LLMs%
  }{ThinQuant: Scalable Rotation Learning for Weight and Activation Quantization of LLMs}%
}

\author{%
\makebox[\textwidth][c]{%
\begin{tabular}{c}
Mehdi Makni\thanks{Equal contribution.},\;
Ryan Lucas\footnotemark[1],\;
Rahul Mazumder \\[0.35em]
Operations Research Center \\
Massachusetts Institute of Technology
\end{tabular}%
}%
}

\begin{document}
\maketitle

\begin{abstract}
Learned rotations play an important role in enabling low-bit weight and activation quantization of
large language models by smoothing outliers in the activation distribution. State-of-the-art approaches include gradient-based procedures such as SpinQuant and computationally friendlier gradient-free
approaches such as DartQuant, but both remain hard to scale to the largest architectures.
To address the computational bottlenecks in gradient-free rotation learning, we introduce two ideas for efficiency, (i) a data selection procedure which reduces the required number of calibration data points, and (ii) an exact reduction of the associated optimization on this reduced calibration set. Our data selection procedure exploits the geometric structure of the convex hull of the activations. Using this idea, we show that a carefully selected calibration set with several orders of magnitude fewer activations than state-of-the-art rotation-based methods can match their performance in low-bit quantization settings. Under this extreme data efficiency, the selected activations span an
$r$-dimensional subspace with $r<d$, making optimization over a $d\times d$ rotation equivalent to optimizing a
$d\times r$ matrix on the Stiefel manifold. We solve this reduced problem using an efficient ADMM algorithm that iteratively
employs thin matrix updates at every step, hence the name ThinQuant.
For Llama-3-70B with W4A4KV4 quantization, ThinQuant completes the entire rotation calibration in under 12 minutes and achieves a WikiText-2 perplexity of 5.63, compared with 7.55 for DartQuant, which requires 111 minutes. Unlike SpinQuant and DartQuant, ThinQuant also scales to Llama-3.1-405B on a single H200 GPU, completing rotation calibration in just over 2 hours and achieving WikiText-2 perplexity of 2.97 at W4A4, compared with 3.48 for GPTAQ+QuaRoT.
\end{abstract}

\section{Introduction}

State-of-the-art large language models (LLMs) contain billions of parameters and process long token sequences, leading to large memory footprints and slow inference time.  As a result, post-training quantization (PTQ), where a pretrained checkpoint is compressed without retraining from scratch, has become a central tool for making LLM inference more efficient \citep{dettmers2022llmint8}. Weight-only PTQ (e.g., INT8 or INT4 weights with higher-precision activations) substantially reduces memory and storage, and methods such as GPTQ \citep{frantar2023gptq}, GPTAQ \citep{li2025gptaqefficientfinetuningfreequantization}, and ADMM-Q \citep{lucas2026admmqimprovedhessianbasedweight} achieve near-lossless accuracy even at 4-bit weight precision. However, weight-only schemes leave the dominant computation—matrix--vector products during inference—executed in higher precision (e.g. FP16 or BF16). Joint weight \textit{and} activation quantization gives greater end-to-end speedups, since the core matrix multiplications can be executed by INT8 or INT4 tensor cores, which offer higher peak throughput than FP16 on modern GPUs.

Joint weight and activation quantization is substantially harder than weight-only quantization. Weights are fixed after training and can be compressed offline using expensive optimization procedures. The resulting quantized tensors are then stored and reused at every inference step. Activations, by contrast, arrive at runtime, are different for every input, and must be quantized on-the-fly with a cheap deterministic rule, typically round-to-nearest (RTN) on a uniform grid. Moreover, it has been observed that for activations a small subset of input channels, often fewer than \(1\%\) of dimensions, carry values that are orders of magnitude larger than the remaining channels \citep{dettmers2022llmint8,xiao2023smoothquant}. 

Rotations address this by changing the coordinate system of the network before quantization which does not change the full precision outputs, but makes the activations substantially easier to quantize. For an orthogonal transformation \(R\in\mathbb{R}^{d\times d}\), a linear map can be equivalently written as $Wx=(WR^{\mathsf T})(Rx)$,
so that activations are represented in the rotated basis while the inverse transformation is absorbed into the surrounding weights. Modern LLM quantization methods exploit this invariance at both \emph{global} and \emph{local} scales \citep{liu2025spinquant,shao2025dartquant}. A global change of basis uses a shared rotation across the hidden representation of multiple layers. Local changes of basis instead act within individual layers or modules, for example inside attention or feed-forward blocks, and can therefore use different rotations at different locations in the network. In both cases, the transformations can be absorbed into adjacent linear operators so that the full-precision network remains functionally identical, while its weights and activations can have substantially more favorable distributions for low-bit quantization. Moreover, rotations can be incorporated into quantized inference pipelines to accelerate inference runtime; we refer the reader to \citep{liu2025spinquant}, Section 4.5.

SpinQuant learns rotations by optimizing an end-to-end loss through the
quantized network using a straight-through estimator \citep{liu2025spinquant}. While effective, this requires
repeated forward and backward passes through the \emph{full} network, making optimization
expensive and potentially prone to overfitting to the fine-tuning objective.
DartQuant substantially reduces this cost by replacing end-to-end training with
a calibration objective on the activations
\citep{shao2025dartquant}. However, this still requires optimizing against a very large collection of activations, which leads to large computational cost. In its standard calibration setting, DartQuant uses a sample $\alpha$ of the tokens generated by \(N=128\times 2048\approx 2.6\times 10^5\) samples from each layer. Across \(L\) transformer blocks, the rotation is therefore calibrated using \(\alpha NL\) activation vectors, where \(\alpha\) is the sampling ratio, set to $10\%$ by default. For Llama-3-70B, \(L=80\), this is on the order of millions of
activation vectors. ThinQuant, using geometry-aware activation selection, takes this to extreme efficiency: using only a few hundred activation vectors from the entire network, ThinQuant gives dramatic savings in disk storage ($\approx 2000\times$) and runtime ($\approx10\times$) for Llama-3-70B with performance comparable to more expensive state-of-the-art methods (e.g. SpinQuant, DartQuant), see Figure \ref{fig:pareto} and Table \ref{tab:peak-disk}.

\begin{wrapfigure}{r}{0.5\textwidth}
\vspace{-6mm}
\includegraphics[width=\linewidth]{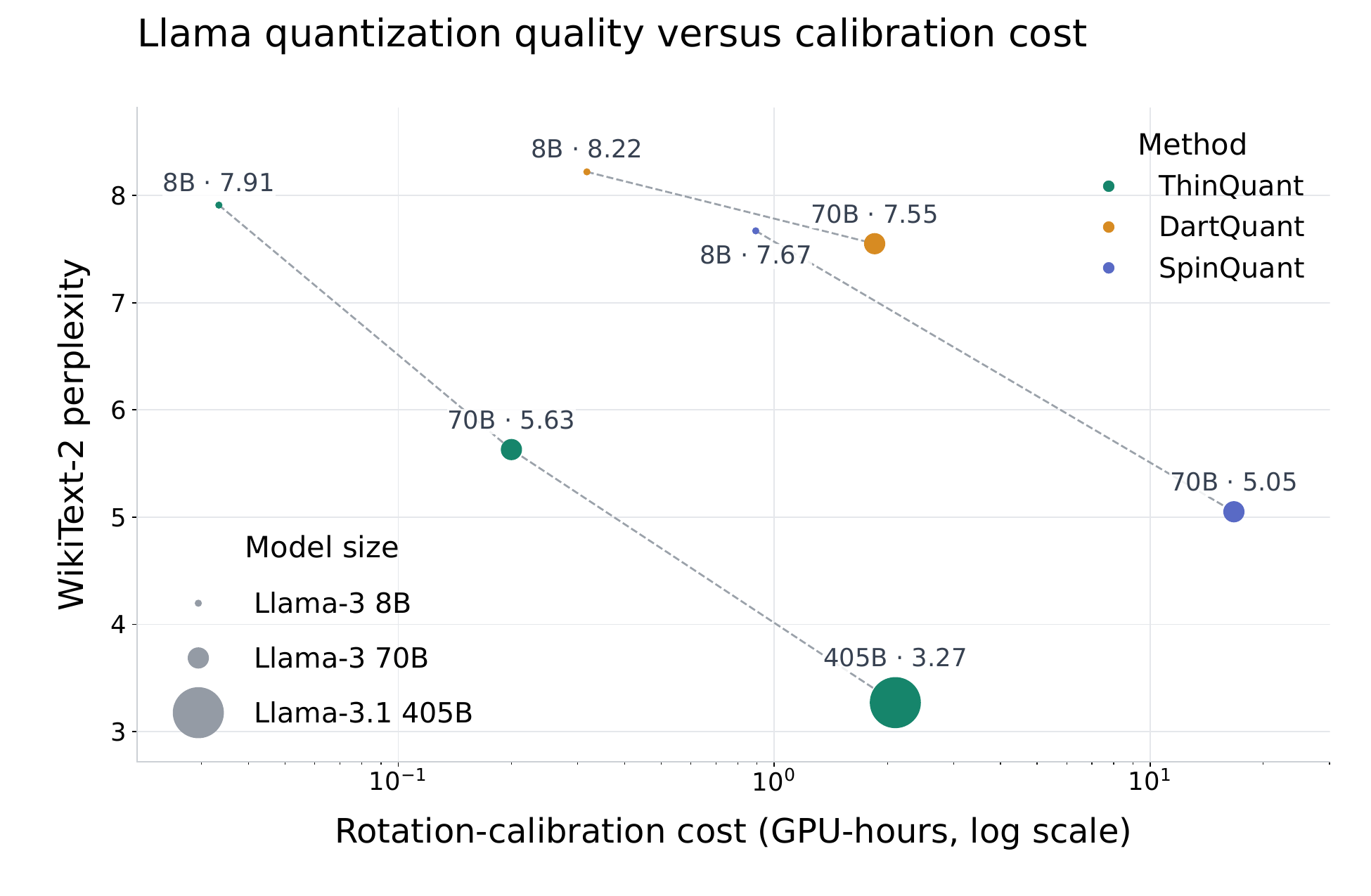}
\vspace{-8mm}
\caption{\small Comparison of Wikitext-2 perplexity versus calibration cost for different rotation-based methods on LLaMA-3 series models at W4A4KV4.
}
\label{fig:pareto}
\end{wrapfigure}

Using our smart activation selection procedure, ThinQuant can operate with a very small calibration dataset, mitigating the shortcomings of existing gradient-free  methods such as DartQuant. We briefly provide some intuition for our selection procedure (see Section \ref{sec:selection} for details). We mitigate the quantization error caused by large outliers in the activations  space by controlling the $\ell_\infty$ norm of the activations after rotation, similar to prior work such as DartQuant and QuaRot \citep{ashkboos2024quarot}. Our key observation is that these
worst-case magnitudes are determined by the \textit{convex hull}
of the activations, rather than by the full activation distribution.
For a calibration set $\cA=\{x_1,\ldots,x_N\}\subset\R^d$, the extreme points of the symmetric convex hull
$K=\operatorname{conv}(\mathcal A\cup(-\mathcal A))$ captures the
largest absolute projection in every direction. Its extreme points therefore suffice
to determine the worst-case objective, while
the interior activations can be discarded without changing it.
Our data selection procedure thus selects geometrically
prominent extreme points in the activations across many possible bases, rather than
only those that are outliers in the original coordinates. In practice,
this geometric selection lets us replace hundreds of thousands
of calibration tokens per layer with a tiny subset of activations.

Given this reduced number of data points, we show that loss functions commonly used to control quantization error (e.g. $\ell_{\infty}$, kurtosis) admit an exact reduction to a smaller optimization problem. Generally, the budget $B$ we use for activation vectors (order 100) is much smaller than the hidden dimension $d$ (order thousands). Let \(U\in\mathbb{R}^{d\times B}\) collect the selected activation
vectors, and let \(r=\operatorname{rank}(U)\leq B\). Writing a thin
factorization \(U=QS\), the columns of
\(Q\in\mathbb{R}^{d\times r}\) are orthonormal, so
\(Q\in\mathrm{St}(d,r)
=\{Y\in\mathbb{R}^{d\times r}:Y^{\mathsf T}Y=I_r\}\).
Since the calibration objective depends on the full rotation
\(R\in\mathbb{R}^{d\times d}\) only through \(RU=(RQ)S\), it is determined
entirely by the \(d\times r\) matrix \(RQ\). We show that optimization over
the full orthogonal group is therefore exactly equivalent to optimization
over the much smaller Stiefel manifold \(\mathrm{St}(d,r)\). Hence ThinQuant, with a few hundred selected activations, replaces a dense \(d\times d\) rotation by a
\(d\times r\) variable with \(r \ll d \), reducing the effective optimization
dimension from quadratic to linear in \(d\), with no additional approximation
beyond the initial activation selection. We summarize our \textbf{contributions} as follows:
\vspace{-2mm}
\begin{itemize}[leftmargin=1em,topsep=1pt,itemsep=1pt]
    \item \textbf{Reducing the calibration data.}
    We develop a geometric activation-selection procedure to select activations that are frequently exposed as extreme
    along random Gaussian directions. This reduces the calibration set required for rotation learning by several
    orders of magnitude, lowering activation storage and the cost of objective evaluations. We support this selection approach with a
    dimension-free bound on the expected discrepancy between the full and reduced
    calibration objectives under Haar-random rotations
    (Theorem~\ref{thm:haar-selection}).

    \item \textbf{Reducing the optimization dimension.}
    We show that this reduced calibration set leads to an exact reduction
    of the rotation-learning problem. For selected activations spanning
    an $r$-dimensional subspace with $r<d$, optimizing a $d\times d$
    rotation is equivalent to optimizing a $d\times r$ matrix on the
    Stiefel manifold, with no additional approximation beyond activation
    selection (Theorem~\ref{thm:reduction}). We exploit this equivalence
    through an ADMM algorithm with thin matrix factorizations, reducing optimization
    memory and per-iteration cost without requiring forward or backward
    passes through the model during optimization.

    \item \textbf{Efficient rotation learning at scale.}
    Combining these reductions, ThinQuant achieves highly competitive
    perplexity and downstream accuracy on Llama-2 and Llama-3 under
    W4A4KV4 quantization at substantially lower calibration cost.
    On Llama-3-70B, ThinQuant reduces total rotation-calibration time
    from $111$ to approximately $12$ minutes relative to DartQuant,
    while improving WikiText-2 perplexity from $7.55$ to $5.63$
    (Table~\ref{tab:results}). Moreover, to demonstrate the extreme scalability of our method, we employ ThinQuant on Llama-3.1-405B. It runs in just over 2 hours on a single H200 GPU, and achieves WikiText-2 perplexity at W4A4 (2.97 vs 3.48 of GPTAQ+QuaRoT), where other optimization-based rotation learning methods fail to scale under similar hardware constraints.
\end{itemize}

\vspace{-5mm}

\section{Related work}

\paragraph{Post-training quantization.}
Weight-only quantization reduces parameter storage and memory traffic
while retaining higher-precision activations. GPTQ \citep{frantar2023gptq} uses approximate second-order information for
weight quantization. GPTAQ \citep{li2025gptaqefficientfinetuningfreequantization} extends GPTQ with a different calibration scheme that uses extra
closed-form weight updates to compensate for discrepancies between
full-precision and quantized layer inputs. AWQ uses activation
statistics to protect salient weight channels \citep{lin2024awq}, and
ADMM-Q uses an ADMM-based method for weight-only quantization
\citep{lucas2026admmqimprovedhessianbasedweight}.
Beyond weight-only methods, joint weight and activation quantization additionally enables
low-precision matrix multiplications, but must accommodate large
outlier values in activation distributions
\citep{dettmers2022llmint8,xiao2023smoothquant}.
SmoothQuant addresses this difficulty through function-preserving
rescaling, transferring quantization difficulty from activations to
weights \citep{xiao2023smoothquant}. OmniQuant further learns clipping
parameters and equivalent transformations through block-wise
reconstruction \citep{shao2024omniquant}.
Such equivalent transformations change the representation presented
to the quantizer without changing the full-precision network output.
Our method also falls within this post-training setting, targeting
joint low-bit weight and activation quantization. Specifically, we
focus on function-preserving rotations, and discuss rotation-based
methods below.
\vspace{-2mm}

\paragraph{Rotation-based quantization.}
Rotation-based methods provide another way to obtain
representations that are more amenable to low-bit quantization.
QuIP uses random orthogonal preprocessing to improve weight and Hessian
incoherence for weight-only quantization \citep{chee2023quip}.
QuaRot applies fixed randomized Hadamard rotations to support weight,
activation, and KV-cache quantization \citep{ashkboos2024quarot}, and is generally the fastest method since it is calibration-free. 
SpinQuant improves substantially on the quality of QuaRot by instead learning rotations by optimizing an end-to-end
quantized-network loss, using Cayley updates to maintain orthogonality
\citep{liu2025spinquant}, but can be slow since it requires full backward passes through the network.
DartQuant reduces this cost by replacing end-to-end optimization with
a direct activation-distribution objective, combining an exponential loss
with a QR decomposition at every step \citep{shao2025dartquant}. Our work builds on these function-preserving rotation frameworks and
focuses on reducing the data and optimization dimensions required
for calibration of the rotation.

\vspace{-2mm}

\section{Method}

Let $X_{\ell}\in\R^{T\times d}$ contain the calibration activations associated with layer $\ell$, with row $t$ written as $x_{\ell,t}^{\mathsf T}$. We adopt the function-preserving rotation placements of prior work \citep{ashkboos2024quarot,liu2025spinquant}, which include a global change of basis $R_1$ and a local change of basis $R_{2,\ell}$ for each layer. The rotation $R_1\in\SO(d)$ defines a single basis shared across all layers. Namely at every layer  $X_\ell$ is represented as $\widetilde X_\ell=X_\ell R_1^{\mathsf T}$ for $\ell=1,\ldots,L$, with the change of basis inverted only at the output as $X_L=\widetilde X_LR_1$. In contrast, each $R_{2,\ell}$ acts only locally within layer $\ell$, with $\widetilde X_\ell=X_\ell R_{2,\ell}^{\mathsf T}$ and $X_\ell=\widetilde X_\ell R_{2,\ell}$. In both cases, the rotations can be absorbed into the surrounding weight matrices, including the final inverse of $R_1$ into the output projection, preserving the full-precision network function without introducing additional online matrix multiplications. Learning these rotations nevertheless leads to large-scale calibration problems. For $R_1$, a single rotation must be calibrated jointly against activations from every layer, so the amount of calibration data grows with the entire network. For $R_{2,\ell}$, the calibration data are local to a single layer, but the rotation itself is still high-dimensional and a separate optimization problem must be solved for every $\ell=1,\ldots,L$. 

ThinQuant addresses these costs with two complementary reductions.
First, for learning both $R_1$ and $R_{2, \ell}$, we select a small set of informative activations
(Section~\ref{sec:selection}), reducing both calibration storage
and the amount of data processed at every optimization step.
Second, we exploit the low-dimensional span of these activations
to reduce the size of the optimization problem
(Section~\ref{sec:reduction}). Instead of updating a dense
$d\times d$ rotation, we optimize a $d\times r$ matrix, where
$r\ll d$ is the rank of the selected activation matrix.
This reformulation is exact for the selected set and replaces
expensive square-matrix orthogonality updates with much cheaper
thin matrix factorizations. Our ADMM solver
(Section~\ref{sec:solver}) operates on this smaller representation to
reduce optimization memory and per-iteration cost.

\vspace{-1mm}
\subsection{Selecting a reduced calibration set}
\label{sec:selection}

\vspace{-1mm}
\paragraph{Rotation calibration.}
We now turn to the rotation calibration problem. We remark that collecting and storing
the full activation matrices can be prohibitively expensive for large
models. For example, DartQuant \citep{shao2025dartquant} explicitly stores a sample of massive calibration
activations, requiring approximately
$960$\,GiB of disk storage for LLaMA-3 70B in their default setting used in our experiments
(Table~\ref{tab:peak-disk} in the appendix). Our goal is therefore to
select a small subset of activations that approximately preserves the
rotation objective \eqref{eq:full-objective}. For $R_1$, we consider activations from all layers, $\cA_1=\bigcup_{\ell=1}^{L}\{x_{\ell,t}^{(1)}:1\leq t\leq T\}$, whereas for $R_{2,\ell}$ we use only the corresponding layer, $\cA_{2,\ell}=\{x_{\ell,t}^{(2)}:1\leq t\leq T\}$. Apart from this choice of calibration set, the method is applied identically to $R_1$ and $R_{2, \ell}$. We therefore write $\cA=\{x_1,\ldots,x_N\}\subset\R^d$ for either collection, $X=[x_1,\ldots,x_N]\in\R^{d\times N}$ for its activation matrix, and $R\in\SO(d)$ for the corresponding rotation. Thus $N=LT$ for $R_1$ and $N=T$ for each $R_{2,\ell}$. For either rotation, a natural worst-case objective including all calibration data is:
\begin{equation}
    \min_{R\in\SO(d)} F(R),
    \qquad
    F(R)
    =
    \|RX\|_\infty
    =
    \max_{x\in\cA}\|Rx\|_\infty,
    \label{eq:full-objective}
\end{equation}
where $\|\cdot\|_\infty$ denotes the largest absolute entry. Geometrically,
the goal is to find an orthonormal basis in which no calibration activation develops a
large coordinate. Uniform quantization is governed by the dynamic
range of the activation, so the largest coordinate directly leads to a coarser
quantization resolution. The $\ell_\infty$ norm therefore
provides a natural worst-case surrogate for suppressing activation
outliers and reducing quantization error
\citep{wei2023outliersuppressionaccuratequantization,
xiao2023smoothquant,chee2023quip,ashkboos2024quarot}. Our goal is to solve \eqref{eq:full-objective} without requiring the full calibration data. Specifically, given an activation budget $B\ll N$, we want a subset $\mathcal I\subseteq[N]$ with $|\mathcal I|=B$ such that
$F_{\mathcal I}(R)=\max_{i\in\mathcal I}\|Rx_i\|_\infty\approx F(R)$.
The subset $\mathcal{I}$ is selected before rotation optimization, so it must
represent this surrogate as the basis $R$ changes, rather than only
in the original coordinate system.  
\vspace{-2mm}
\paragraph{A geometric view of calibration data selection.} We first give some intuition on how the calibration data $\mathcal{A}$ in \eqref{eq:full-objective} might naturally be reduced. First, notice that the worst-case objective in \eqref{eq:full-objective} is determined,
for any fixed rotation $R$, by activations attaining
$\max_{x\in\cA}|r_k^{\mathsf T}x|$ for one of the rows
$r_k^{\mathsf T}$ of $R$. However, since the rows of $R$ change during
optimization (and the optimal $R$ is not known in advance), the maximizer can also change, so
activations that are not maximizers in the original basis may become so after rotation. Thus, rather than retaining
only activations that are large in the original basis,
we want a small subset that retains the  activations that may become maximizers during
rotation optimization.

Classical convex geometry provides a precise language for describing
these extreme directions \citep{ziegler1995lectures}. Define the
symmetric convex hull
$K=\conv\bigl(\cA\cup(-\cA)\bigr)$ and its support function
$h_K(v)=\max_{z\in K}v^{\mathsf T}z
=\max_{x\in\cA}|v^{\mathsf T}x|$, defined for every $v\in\R^d$.
Since $\cA$ is finite and $K$ is a polytope, maximizing the linear
functional $x\mapsto v^{\mathsf T}x$ over $K$ attains its optimum at an
extreme point of $K$ for every $v$ \citep{bertsimas1997introduction}. The candidate maximizers (in $x$) are therefore confined to
$\Ext(K)\subseteq\cA\cup(-\cA)$, where $\Ext(\cdot)$ is the set of
extreme points of the polytope. This holds independently of $v$, so a single finite
set of signed activations suffices to evaluate $h_K$ in every direction.
Writing the rows of a rotation $R$ as
$r_1^{\mathsf T},\ldots,r_d^{\mathsf T}$, we can thus write
\eqref{eq:full-objective} equivalently as:
\begin{equation}
 \min_{R\in\SO(d)} F(R)
  =
  \max_{1\leq k\leq d} h_K(r_k),
  \label{eq:hull}
\end{equation}
The key observation here is that changing $R$
changes the viewing direction, but not the underlying polytope
$K$. Figure~\ref{fig:switch} illustrates this.
 The directions $r_1,r_2,r_3$ expose different activations
    $x_1,x_2,x_3$, respectively, but each support value can only be attained at an
    extreme point of the same convex hull. Thus, the relevant boundary point
    can change as the rotation changes, but activations on the interior of $K$
    can never obtain an extreme point. This explains why the full calibration set can be redundant. An activation which is not an extreme point
cannot become extreme along any direction. Any
subset $\mathcal A_{\mathcal{I}}\subseteq\cA$ satisfying
$\conv\bigl(\mathcal A_{\mathcal{I}}\cup(-\mathcal A_{\mathcal{I}})\bigr)=K$ preserves $F(R)$
for every rotation $R$. This also connects rotation calibration to the literature on geometric coresets and, in particular, $\varepsilon$-kernels, where small subsets are constructed to approximately preserve the width of the convex hull in all directions
\citep{agarwal2004approximating}.

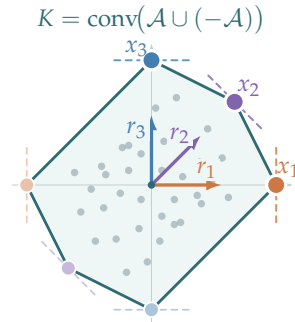
\begin{wrapfigure}[14]{r}{0.4\textwidth}
\vspace{-2\intextsep}
\centering
\begin{tikzpicture}[
x=.55cm,y=.55cm,
line cap=round,
line join=round
]

\definecolor{prettyhull}{HTML}{2D6A73}
\definecolor{prettyfill}{HTML}{EFF6F6}
\definecolor{prettycloud}{HTML}{9AAEB3}
\definecolor{prettyaxis}{HTML}{C2D0D3}
\definecolor{dirone}{HTML}{CC7744}
\definecolor{dirtwo}{HTML}{8064AE}
\definecolor{dirthree}{HTML}{447EB0}

\path[use as bounding box]
  (-3.7,-3.5) rectangle (3.7,4.35);

\fill[prettyfill] \HullPath;

\draw[prettyaxis,line width=.35pt,-{Latex[length=1.3mm]}]
  (-3.35,0)--(3.40,0);
\draw[prettyaxis,line width=.35pt,-{Latex[length=1.3mm]}]
  (0,-3.35)--(0,3.40);

\foreach \xx/\yy in \InteriorCoords {
  \fill[prettycloud,opacity=.70] (\xx,\yy) circle (1.4pt);
  \fill[prettycloud,opacity=.70] (-\xx,-\yy) circle (1.4pt);
}

\draw[prettyhull,line width=1.05pt] \HullPath;

\foreach \ang/\h/\col in {
  0/3/dirone,
  45/2.828427/dirtwo,
  90/3/dirthree%
}{
  \begin{scope}[rotate=\ang]
    \draw[\col!80,line width=.75pt,dash pattern=on 2.3pt off 2pt]
      (\h,-.90)--(\h,.90);
    \draw[\col!40,line width=.65pt,dash pattern=on 2.3pt off 2pt]
      (-\h,-.90)--(-\h,.90);
  \end{scope}
}

\foreach \xx/\yy/\col in {
  -3/0/dirone,
  -2/-2/dirtwo,
  0/-3/dirthree%
}{
  \filldraw[fill=\col!45,draw=white,line width=.7pt]
    (\xx,\yy) circle (2.8pt);
}

\foreach \ang/\col/\lab in {
  0/dirone/r_1,
  45/dirtwo/r_2,
  90/dirthree/r_3%
}{
  \draw[\col,line width=1.25pt,-{Latex[length=2.1mm,width=1.5mm]}]
    (0,0)--(\ang:1.70);
  \node[text=\col,fill=prettyfill,inner sep=1pt,font=\footnotesize]
    at ({1.33*cos(\ang)-.36*sin(\ang)},
        {1.33*sin(\ang)+.36*cos(\ang)}) {$\lab$};
}

\fill[prettyhull] (0,0) circle (1.4pt);

\foreach \xx/\yy/\col in {
  3/0/dirone,
  2/2/dirtwo,
  0/3/dirthree%
}{
  \fill[white] (\xx,\yy) circle (4.2pt);
  \fill[\col] (\xx,\yy) circle (3pt);
}

\node[text=dirone,font=\footnotesize,inner sep=1pt]
  at (3.30,.38) {$x_1$};
\node[text=dirtwo,font=\footnotesize,inner sep=1pt]
  at (2.34,2.34) {$x_2$};
\node[text=dirthree,font=\footnotesize,inner sep=1pt]
  at (-.36,3.28) {$x_3$};

\node[text=prettyhull,font=\footnotesize,anchor=south,inner sep=0pt]
  at (0,3.6)
  {\small $K=\operatorname{conv}\!\bigl(\mathcal A\cup(-\mathcal A)\bigr)$};

\end{tikzpicture}
\vspace{-2mm}
\caption{\small
For any direction $r$, the support value
$h_K(r)=\max_{x\in\cA}|r^{\mathsf T}x|$
is attained at an extreme point of $K$. Hence, as the rotation changes,
only the exposed boundary activations can determine $F(R)$ and interior
activations do not determine the objective.}
\label{fig:switch}
\end{wrapfigure}

\vspace{-3mm}
\paragraph{Selection algorithm.}
Although the relevant calibration information is concentrated on extreme
points, recovering all extreme points is unattractive in our setting. For general polyhedra, deciding if a partial
vertex list omits another vertex is NP-complete
\citep{khachiyan2008generating}. Even for a finite number of points, exact identification of all extreme points would require
solving a sequence of high-dimensional convex-hull membership problems. Similarly, $\varepsilon$-kernels, which approximate the
directional width of a point set simultaneously over all directions,
have worst-case size
$\Theta\!\left(\varepsilon^{-(d-1)/2}\right)$
for preserving directional widths within a relative
$(1-\varepsilon)$ factor \citep{agarwal2004approximating}. With hundreds of thousands of activation vectors in dimensions of several
thousand, these are considerably more expensive than the lightweight
selection procedure we want.

A natural and more lightweight approach is to query the support function $h_K(v)$ from \eqref{eq:hull} along random directions.
This idea has precedent in randomized extreme-point detection. In particular, \cite{damle2017geometric} generate Gaussian linear
functionals, record the points attaining their extrema, and use
selection frequency to identify geometrically prominent extreme points. We apply the same random-projection idea to the symmetric hull $K=\conv(\cA\cup(-\cA))$. We give the procedure in Algorithm \ref{alg:directional-probes}. We compare the results from this procedure with $\ell_\infty$-based selection in
Table~\ref{tab:selection-ablation} in the appendix, where we retain
the $B$ activations with the largest $\ell_\infty$ norm either in
the original basis or after an initial Hadamard rotation.

\begin{algorithm}[tb]
\caption{Activation selection}
\label{alg:directional-probes}
\small
\begin{algorithmic}[1]
\Require Activations $\cA=\{x_1,\ldots,x_N\}$,
         number of random directions $M$, total activation budget $B$
\State $g_1,\ldots,g_M \overset{\mathrm{iid}}{\sim} \cN(0,I_d)$
       \Comment{Gaussian directions}
\State $i_j \gets \argmax_{1\leq i\leq N}
        |g_j^{\mathsf T}x_i|,\quad j=1,\ldots,M$
       \Comment{Extreme point?}
\State $e_i \gets \sum_{j=1}^{M}\mathbf 1\{i_j=i\},
        \quad i=1,\ldots,N$
       \Comment{How often extreme?}
\State $\mathcal I \gets \operatorname{Top}_B(e)$
       \Comment{Global ranking}
\State \Return $\{x_i:i\in\mathcal I\}$
\end{algorithmic}
\end{algorithm}

\paragraph{Interpretation of the algorithm.}
We now give an interpretation of Algorithm~\ref{alg:directional-probes} in terms of the calibration objective. Notice that \eqref{eq:hull} describes the calibration objective under the best possible change of basis. Rather than requiring a small subset to preserve this objective for every possible rotation (which is clearly intractable), consider how well it is preserved on average over
$R\sim\operatorname{Haar}(\SO(d))$, i.e., for a randomly sampled Haar rotation:
\begin{equation}
\mathbb E_{R\sim\operatorname{Haar}(\SO(d))}
\left[F(R)\right]
=
\mathbb E_{R\sim\operatorname{Haar}(\SO(d))}
\left[
\max_{1\leq k\leq d} h_K(r_k)
\right].
\label{eq:expected-hull}
\end{equation}
Algorithm~\ref{alg:directional-probes} can be interpreted as selecting a small subset of the
activations with which to approximate the objective in \eqref{eq:expected-hull}.
For the selected subset $\mathcal{I}\subseteq[N]$, let
$K_{\mathcal{I}}=\operatorname{conv}\{\pm x_i:i\in\mathcal{I}\}$ and write $
F_{\mathcal{I}}(R)
=
\max_{1\leq k\leq d}h_{K_{\mathcal I}}(r_k)$
as the corresponding reduced calibration objective.
Since $K_{\mathcal I}\subseteq K$, we have
$F_{\mathcal I}(R)\leq F(R)$ for every rotation $R$.
Thus the subset objective is
$\mathbb E_R[F_{\mathcal I}(R)]$, and we can measure the quality
of the selected calibration set via
$\mathbb E_R[F(R)-F_{\mathcal I}(R)]$, which measures how well the selected activations represent
the full calibration objective in \eqref{eq:full-objective}
under a random change of basis.

Theorem~\ref{thm:haar-selection} gives a bound on this discrepancy. Specifically, let $q(\mathcal I)$ denote the probability that, along a random Gaussian direction, the maximum support value is attained by an activation that is not contained in the selected set $\mathcal I$. The theorem shows that this probability directly controls the expected difference between the full and reduced calibration objectives. Thus, when Algorithm~\ref{alg:directional-probes} selects activations that attain the directional maximum for most random Gaussian directions, the reduced  objective closely approximates \eqref{eq:full-objective} on average over random changes of basis, without requiring recovery of the entire convex hull.

\begin{theorem}[Expected Approximation Guarantee]
\label{thm:haar-selection}
Let $\mathcal A=\{x_1,\ldots,x_N\}\subset\mathbb R^d$, with $d\ge2$,
and identify activations that coincide up to sign.
For any fixed nonempty subset $\mathcal I\subseteq[N]$ returned by Algorithm \ref{alg:directional-probes}, define:
\[
q(\mathcal I)
=
\Pr_{g\sim\mathcal N(0,I_d)}
\left[
\operatorname*{arg\,max}_{i\in[N]}|g^{\mathsf T}x_i|
\notin\mathcal I
\right],
\]
Then, for $R\sim\operatorname{Haar}(\SO(d)):$
\begin{equation}
\mathbb E_R\!\left[F(R)-F_{\mathcal I}(R)\right]
\le
\left(\max_{i\in[N]}\|x_i\|_2\right)
\sqrt{3q(\mathcal I)}.
\label{eq:expected-selection-bound}
\end{equation}
\end{theorem}

The proof is given in Appendix~\ref{sec:thm_proof}.
The main idea is to relate the Gaussian directions used by
Algorithm~\ref{alg:directional-probes} to the coordinate directions
generated by a random rotation. If
$R\sim\operatorname{Haar}(\SO(d))$, then each row direction $r_k$
is marginally uniform on the unit sphere. Likewise, for
$g\sim\mathcal N(0,I_d)$, the normalized vector $g/\|g\|_2$
is uniform on the unit sphere. Thus, projecting an activation
onto a normalized Gaussian direction is distributionally
equivalent to an individual coordinate after a Haar
rotation. Although the rows of $R$ are dependent through
orthogonality, the proof bounds the objective gap by
a sum of coordinate errors, whose expectation depends
only on these marginal distributions.

\begin{wrapfigure}{r}{0.45\textwidth}
\vspace{-\intextsep}
\centering
\includegraphics[width=\linewidth]{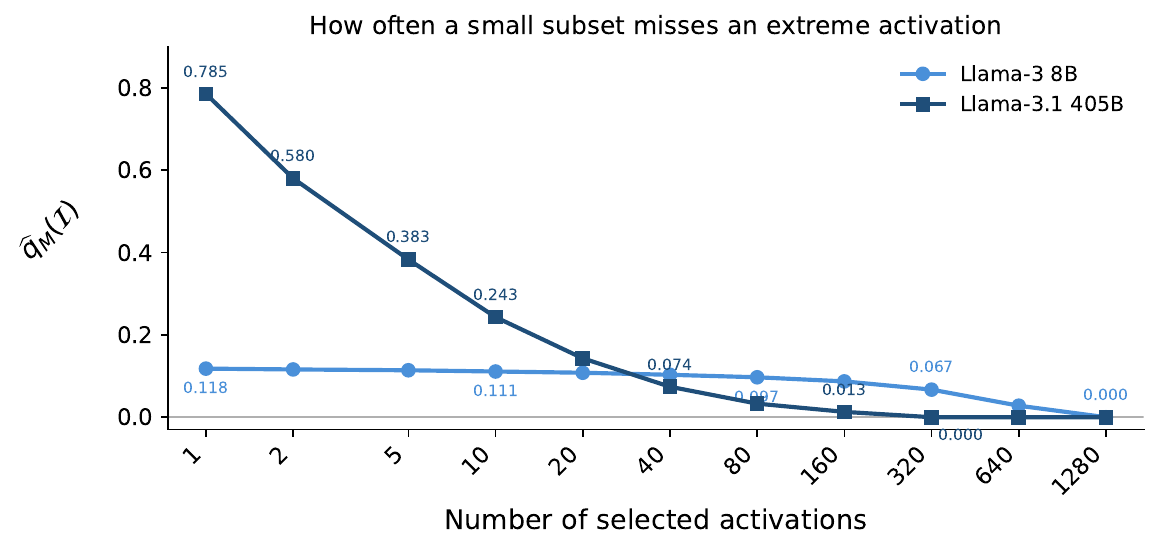}
\caption{\small
Empirical estimate $\widehat q_M(\mathcal I)$ of the miss probability
$q(\mathcal I)$ in Algorithm \ref{alg:directional-probes} as a function of the activation budget $B$.}
\label{fig:qhat}
\end{wrapfigure}
An appealing feature of Theorem~\ref{thm:haar-selection} is that the
quantity controlling the bound is directly estimated by
Algorithm~\ref{alg:directional-probes}. Indeed, for the selected set
$\mathcal I$, the counts $e_i$ gives the empirical estimate:
\[
\widehat q_M(\mathcal I)
=
1-\frac{1}{M}\sum_{i\in\mathcal I}e_i,
\]
which is precisely the fraction of sampled Gaussian directions whose
maximizing activation is not contained in $\mathcal I$. As
$M\rightarrow\infty$, these empirical frequencies converge to
their population probabilities, and hence
$\widehat q_M(\mathcal I)\rightarrow q(\mathcal I)$. We report $\widehat q_M(\mathcal I)$ for both Llama-3.1-405B and Llama-3-8B at $M=8192$. We see that the counts are highly
concentrated. Selecting a single activation already
reduces the empirical miss probability $\widehat q_M(\mathcal I)$ to $0.78$
and $0.12$, respectively, and it reaches zero on both models at $B=1280$ activations. For Llama-3-8B and Llama-3.1-405B,
this is already far less than $0.1\%$ of $\mathcal A$. Finally, given the selected index set $\mathcal I$ returned by Algorithm \ref{alg:directional-probes}, we collect the selected
activations as columns of the reduced calibration matrix:
\begin{equation}
    U
    =
    X_{:,\mathcal I}
    =
    [x_i]_{i\in\mathcal I}
    \in\R^{d\times B}.
    \label{eq:selected-calibration-matrix}
\end{equation}
We replace the calibration matrix $X$ with the selected
matrix $U=[x_i]_{i\in\mathcal I}\in\R^{d\times B}$, where
$|\mathcal I|=B$:
\begin{equation}
    \min_{R\in\SO(d)} F_U(R),
    \qquad
    F_U(R)
    =
    \|RU\|_\infty
    =
    \max_{i\in\mathcal I}\|Rx_i\|_\infty
    \label{eq:selected-objective}
\end{equation}
The calibration procedure thus retains activations that determine
the maximum across many directions, while reducing
the number of calibration vectors from $N$ to $B$. This also
restricts the span of the retained activations to dimension at most
$B$, which we discuss next.


\subsection{An exact reduction in the rotation optimization problem}
\label{sec:reduction}
Once $U$ is fixed, the objective in \eqref{eq:selected-objective}
evaluates $R$ only on the selected activations, all of which lie
in $\operatorname{span}(U)$. Any two rotations that agree on this
subspace therefore produce the same matrix $RU$ and attain the
same loss, regardless of how they act on its orthogonal complement.
Thus, rather than tracking the action of $R$ on all of
$\R^d$, it suffices to track its action on a basis for
$\operatorname{span}(U)$. To make this clear, take an exact QR
decomposition $U=Q_US$, where $r=\rank(U)\leq B$.
Here, $Q_U\in\R^{d\times r}$ has orthonormal columns spanning
$\operatorname{span}(U)$, and $S\in\R^{r\times B}$ contains the
coordinates of the selected activations in this basis.
The matrix $V=RQ_U$ records the images of these basis vectors
under $R$, and hence determines every rotated activation through
$RU=VS$. Moreover, orthogonality of $R$ implies
$V^{\mathsf T}V=Q_U^{\mathsf T}R^{\mathsf T}RQ_U=I_r$, so
$V\in\St(d,r):=\{V\in\R^{d\times r}:V^{\mathsf T}V=I_r\}$.
The following theorem proves the converse when $r<d$,
showing that every $V\in\St(d,r)$ can be written as $RQ_U$
for some $R\in\SO(d)$, so optimizing over $V$ is
equivalent to optimizing over the full rotation $R$.

\begin{theorem}[Exact selected-subspace reduction]
\label{thm:reduction}
Let $U=Q_US\in\R^{d\times B}$ have rank $1\leq r<d$, where
$Q_U\in\R^{d\times r}$, $S\in\R^{r\times B}$, and
$Q_U^{\mathsf T}Q_U=I_r$. Then
$\{RU:R\in\SO(d)\}=\{VS:V\in\St(d,r)\}$. Then for every
continuous function $\ell:\R^{d\times B}\to\R$:
\begin{equation}
    \min_{R\in\SO(d)}\ell(RU)
    =
    \min_{V\in\St(d,r)}\ell(VS).
    \label{eq:exact}
\end{equation}
Every feasible reduced solution lifts to a full rotation with
identical objective value.
\end{theorem}
\begin{proof}
The forward inclusion follows from $V=RQ_U$. Conversely, extend
$Q_U$ and $V$ to orthogonal bases $[Q_U\ Q_\perp]$ and
$[V\ V_\perp]$ and set
$R=[V\ V_\perp][Q_U\ Q_\perp]^{\mathsf T}$. Then $RQ_U=V$.
Because $r<d$, a sign change in one complement column can enforce
$\det R=1$ without changing $RQ_U$. Hence $RU=VS$. Since $\mathrm{SO}(d)$ and $\mathrm{St}(d,r)$ are compact and $\ell$ is continuous, both minima are attained; equality of the feasible image sets therefore implies \eqref{eq:exact}.
\end{proof}
\vspace{-2mm}

Theorem~\ref{thm:reduction} applies directly to the entrywise
$\ell_\infty$ objective in \eqref{eq:selected-objective}, since
$\ell(RU)=\|RU\|_\infty$ is continuous.
It also applies without modification to any continuous surrogate
objective.

\paragraph{Calibration loss.}
In general, directly optimizing the $\ell_\infty$ norm is not attractive, since it requires taking subgradients and can make convergence slow. For optimization, we thus use a fourth-power penalty as a smooth
surrogate for the entrywise $\ell_\infty$ objective.
Since $\|RU\|_\infty^4\leq\|RU\|_4^4$, this provides a smooth
upper bound on the fourth power of the maximum entry.
This choice is motivated by related kurtosis-based approaches
to activation outlier suppression
\citep{akhondzadeh2025kurtailkurtosisbasedllm}. Moreover, we center each rotated activation within the loss
because a common shift of all coordinates changes their absolute
magnitudes but leaves their range unchanged.
This is relevant to the asymmetric activation
quantization used in our experiments, where the quantization
grid need not be centered at zero.
We therefore penalize deviations from each rotated activation's
coordinate mean.
Specifically, let
$C_d=I_d-\frac{1}{d}\mathbf{1}_d\mathbf{1}_d^{\mathsf T}$,
where $\mathbf{1}_d\in\R^d$ is the all-ones vector. We therefore replace the objective in
\eqref{eq:selected-objective} with a surrogate in \eqref{eq:centered-reduced}.
Since the centered loss is still continuous,
Theorem~\ref{thm:reduction} gives:
\begin{equation}
    \min_{R\in\SO(d)}\|C_dRU\|_4^4
    =
    \min_{V\in\St(d,r)}\|C_dVS\|_4^4.
    \label{eq:centered-reduced}
\end{equation}

\subsection{Fast optimization and rotation recovery}
\label{sec:solver}

The reduced problem in \eqref{eq:centered-reduced} replaces
the full rotation with a thin matrix $V$, but the loss remains
coupled to the orthogonality constraint through $C_dVS$.
We use the Alternating Direction Method of Multipliers
(ADMM) \citep{boyd2011admm} to separate these components.
Specifically, introducing $W=VS$ and $Z=C_dW$, we can write the problem as: 
\begin{equation}
    \min_{V,W,Z}\ \|Z\|_4^4
    \qquad
    \text{subject to}\quad
    W=VS,\quad
    Z=C_dW,\quad
    V^{\mathsf T}V=I_r.
    \label{eq:admm-split}
\end{equation}
ADMM alternates between updates of $Z$, $W$, and $V$ -- each of which can be computed via closed-form updates (see Appendix \ref{app:updates}) --
followed by dual updates that track violations of the two
linear constraints.
The $Z$-update handles the fourth-power penalty through
independent proximal updates by solving a cubic polynomial per entry.
The $W$-update balances agreement with $VS$ and the centered
variable $Z$ through a closed-form least-squares step.
Since $C_d$ is a projector, this step requires only column
means rather than a dense matrix inverse.
Finally, the $V$-update enforces orthogonality by solving
an orthogonal Procrustes problem \citep{golub} using a thin $d\times r$ Singular Value Decomposition (SVD).
This update remains exact even when $S$ is not orthogonal,
because $\|VS\|_F^2=\|S\|_F^2$ on the Stiefel manifold.
All iterations operate on the reduced calibration matrices,
so no transformer forward or backward pass is needed during
optimization.
\vspace{-2mm}
\paragraph{Cost.}
Factoring $U$ costs $O(dB^2)$ once for $B\leq d$.
Each iteration costs $O(drB+dr^2)$ and stores
$O(dB+dr+rB)$ numbers.
When $r=B\ll d$, these become $O(dB^2)$ work and
$O(dB)$ storage, replacing a square orthogonality step
with $O(d^3)$ work and $O(d^2)$ rotation state.

After optimization, we recover $R$ using the completion
in the proof of Theorem~\ref{thm:reduction}, with a random
Hadamard rotation for the orthogonal complement.
Completion leaves $RU=VS$ unchanged, although different
completions can act differently on unseen activations
outside $\operatorname{span}(U)$.
\vspace{-3mm}
\section{Experiments}

\begin{table}[t]
  \centering
  \begin{minipage}[t]{0.57\linewidth}
  \vspace{0pt}
  \centering
  \fontsize{8.5pt}{10pt}\selectfont
  \setlength{\tabcolsep}{3.5pt}
  \renewcommand{\arraystretch}{1.08}
  \resizebox{\linewidth}{!}{%
  \begin{tabular}{@{}llrrr@{\hspace{1em}}rr@{}}
    \toprule
    & &
    \multicolumn{3}{c}{\textbf{Quantization performance}}
    &
    \multicolumn{2}{c}{\textbf{Rotation calibration cost}}
    \\
    \cmidrule(lr){3-5}
    \cmidrule(lr){6-7}
    \textbf{Model}
    & \textbf{Method}
    & \textbf{Wiki-2} $\downarrow$
    & \textbf{C4} $\downarrow$
    & \textbf{0-shot} $\uparrow$
    & \textbf{Runtime} $\downarrow$
    & \textbf{Peak GPU} $\downarrow$
    \\
    \midrule

    \cellcolor{white}
    & \texttt{Dense FP16}
    & 6.14 & 9.44 & 65.91
    & & 
    \\
    \rowcolor{darkgreen!10}
    \cellcolor{white}
    & \texttt{ThinQuant}
    & 7.91 & 12.83 & \textbf{59.70}
    & \textbf{1:56} & \textbf{18.5}
    \\
    \cellcolor{white}
    & \texttt{DartQuant}
    & 8.22 & 13.37 & 59.00
    & 19:03 & 37.1
    \\
    \cellcolor{white}\multirow{-4}{*}{\textbf{LLaMA-3 8B}}
    & \texttt{SpinQuant}
    & \textbf{7.67} & \textbf{12.69} & 59.62
    & 53:35 & 20.7
    \\
    \midrule

    \cellcolor{white}
    & \texttt{Dense FP16}
    & 2.86 & 7.17 & 72.70
    & &
    \\
    \rowcolor{darkgreen!10}
    \cellcolor{white}
    & \texttt{ThinQuant}
    & 5.63 & 10.34 & 65.83
    & \textbf{11:57} & \textbf{31.9}
    \\
    \cellcolor{white}
    & \texttt{DartQuant}
    & 7.55 & 15.69 & 58.64
    & 111:00 & 134.9
    \\
    \cellcolor{white}\multirow{-4}{*}{\textbf{LLaMA-3 70B}}
    & \texttt{SpinQuant}
    & \textbf{5.05} & \textbf{9.71} & \textbf{67.87}
    & 16.7h$^\ast$ & 239.0$^\ast$
    \\
    \midrule

    \cellcolor{white}
    & \texttt{Dense FP16}
    & 5.47 & 7.26 & 61.16
    & &
    \\
    \rowcolor{darkgreen!10}
    \cellcolor{white}
    & \texttt{ThinQuant}
    & 6.18 & \textbf{8.41} & \textbf{57.28}
    & \textbf{1:43} & 18.2
    \\
    \cellcolor{white}
    & \texttt{DartQuant}
    & 6.33 & 8.63 & 56.39
    & 20:25 & 75.1
    \\
    \cellcolor{white}\multirow{-4}{*}{\textbf{LLaMA-2 7B}}
    & \texttt{SpinQuant}
    & \textbf{6.12} & 8.49 & 56.98
    & 49:33 & \textbf{17.8}
    \\
    \midrule

    \cellcolor{white}
    & \texttt{Dense FP16}
    & 4.88 & 6.73 & 64.30
    & &
    \\
    \rowcolor{darkgreen!10}
    \cellcolor{white}
    & \texttt{ThinQuant}
    & 5.47 & 7.67 & 61.10
    & \textbf{2:51} & \textbf{21.9}
    \\
    \cellcolor{white}
    & \texttt{DartQuant}
    & 5.51 & 7.71 & 61.35
    & 31:42 & 51.1
    \\
    \cellcolor{white}\multirow{-4}{*}{\textbf{LLaMA-2 13B}}
    & \texttt{SpinQuant}
    & \textbf{5.35} & \textbf{7.61} & \textbf{61.73}
    & 113:32 & 31.5
    \\
    \midrule

    \cellcolor{white}
    & \texttt{Dense FP16}
    & 3.32 & 5.71 & 69.52
    & &
    \\
    \rowcolor{darkgreen!10}
    \cellcolor{white}
    & \texttt{ThinQuant}
    & 3.91 & 6.24 & 66.92
    & \textbf{11:42} & \textbf{33.2}
    \\
    \cellcolor{white}
    & \texttt{DartQuant}
    & 3.91 & 6.22 & 67.18
    & 99:57 & 73.1
    \\
    \cellcolor{white}\multirow{-4}{*}{\textbf{LLaMA-2 70B}}
    & \texttt{SpinQuant}
    & \textbf{3.85} & \textbf{6.19} & \textbf{67.62}
    & 17.2h$^\ast$ & 224.3$^\ast$
    \\
    \bottomrule
  \end{tabular}}
  \caption{
    \footnotesize W4A4KV4 quantization performance and rotation calibration cost.
    The 0-shot column averages nine downstream tasks.
    Dense FP16 denotes the unquantized reference model.
    Calibration-cost columns apply to the quantized methods.
    Runtime measures total rotation-calibration time for ThinQuant,
    DartQuant, and SpinQuant.
    GPU denotes peak GPU process memory during calibration, in GiB.
    Times are in \texttt{mm:ss} unless marked in hours.
    $^\ast$The 70B SpinQuant runs used 4 GPUs, so we report total GPU hours
    and total GPU memory summed across GPUs.
    Best quantized results are \textbf{bold}.
  }
  \label{tab:results}
  \end{minipage}%
  \hfill
  \begin{minipage}[t]{0.405\linewidth}
  \vspace{0pt}
  \centering
  \footnotesize
  \setlength{\tabcolsep}{2.5pt}
  \renewcommand{\arraystretch}{1.02}
  \resizebox{\linewidth}{!}{%
  \begin{tabular}{@{}lrr@{}}
    \toprule
    \textbf{Method}
    & \shortstack{\textbf{LLaMA-3}\\\textbf{70B}}
    & \shortstack{\textbf{LLaMA-3.1}\\\textbf{405B}}
    \\
    \midrule
    \texttt{Dense FP16}
    & 2.86 & 1.44
    \\
    \texttt{GPTQ} $+$ \texttt{QuaRot}
    & 6.04 & 5.82
    \\
    \texttt{GPTAQ} $+$ \texttt{QuaRot}
    & 5.81 & 3.48
    \\
    \rowcolor{darkgreen!10}
    \texttt{GPTQ} $+$ \texttt{ThinQuant}
    & 5.63 & 3.10
    \\
    \rowcolor{darkgreen!10}
    \texttt{GPTAQ} $+$ \texttt{ThinQuant}
    & \textbf{5.46} & \textbf{2.97}
    \\
    \bottomrule
  \end{tabular}}
  \caption{\footnotesize
   Scalability to extremely large architectures.
    WikiText-2 perplexity ($\downarrow$) on LLaMA-3-70B and LLaMA-3.1-405B. 
    The LLaMA-3-70B column is W4A4KV4 with symmetric weights and asymmetric activations.
    The LLaMA-3.1-405B column is W4A4.
    For the 405B comparison, ThinQuant uses the same asymmetric
    quantization as GPTAQ \cite{li2025gptaqefficientfinetuningfreequantization} for both weights and activations.
    The best quantized result in each column is \textbf{bold}. ThinQuant takes 2 hours 7 minutes for full rotation calibration on LLaMA-3.1-405B. We additionally run W4A4KV4 under GPTQ + ThinQuant for LLaMA-3.1-405B, and report the result in Figure \ref{fig:pareto}.
  }
  \label{tab:405b-scalability}
  \end{minipage}
\end{table}

\subsection{Experimental setup}
\label{sect:exp-setup}
\vspace{-0.5em}
\noindent\textbf{Models and baselines.}~~~
We evaluate ThinQuant on LLaMA-2 (7B, 13B, and 70B),
LLaMA-3 (8B and 70B)~\citep{dubey2024llama}. All models are evaluated under W4A4(KV4)
quantization, with weights, activations, and (when applicable) the KV
cache quantized to 4 bits. We use symmetric quantization for
weights and asymmetric quantization for activations and the
KV cache. As in SpinQuant and DartQuant, we use fixed random
Hadamard rotations for $R_3$ and $R_4$ (which are KV cache specific); these rotations are not
optimized and are enabled for every ThinQuant, DartQuant, and
SpinQuant evaluation.
Unless otherwise stated, weight quantization is performed using
GPTQ throughout all experiments, and
is applied after rotation calibration so that GPTQ operates on
the rotated weight matrices. We compare against
DartQuant and
SpinQuant under the same quantization settings. All experiments for all methods are conducted
on a single NVIDIA H200 GPU, with the exception of SpinQuant where we use 4 H200 GPUs and report total GPU hours.

\smallskip
\noindent\textbf{Evaluation.}~~~
We measure language-modeling performance using perplexity
($\downarrow$) on WikiText-2, C4, and Penn Treebank (PTB).
For downstream evaluation, we report the average zero-shot
accuracy ($\uparrow$) across nine tasks using the
LM Evaluation Harness~\citep{gao10256836framework}:
PIQA~\citep{bisk2020piqa},
ARC-Easy and ARC-Challenge~\citep{clark2018think},
HellaSwag~\citep{zellers2019hellaswag},
Winogrande~\citep{sakaguchi2021winogrande},
RTE~\citep{wang2018glue},
OpenBookQA~\citep{mihaylov2018openbookqa},
BoolQ~\citep{clark2019boolq}, and
Social IQa~\citep{sap2019socialiqacommonsensereasoningsocial}.

\smallskip
\noindent\textbf{Calibration.}~~~
Unless otherwise stated, ThinQuant uses $M=8192$ independent
Gaussian directions in
Algorithm~\ref{alg:directional-probes}.
For the shared residual rotation $R_1$, we set the total
activation budget to $B=2L+1$, where $L$ denotes the number
of transformer layers. The rotations are applied at the
down-projection and output-projection layers, so this budget
is sufficient to admit one activation vector from each
projection in every layer, though we select the top $B$ globally as in Algorithm \ref{alg:directional-probes}. For $R_2$, we set the budget to $H$
vectors, where $H$ is the number of transformer heads.
The selected activations are used
to construct the reduced optimization problem, which we
solve using the ADMM procedure in
Section~\ref{sec:solver}.
To assess computational efficiency, we report total rotation-calibration run-time and peak GPU memory in Table \ref{tab:results}, with a more detailed breakdown in Table \ref{tab:peak-disk}.
Memory is reported in GiB.

\paragraph{Results.}
Table~\ref{tab:results} shows that ThinQuant improves or closely matches the performance of DartQuant across all reported evaluation
metrics on LLaMA-2 and LLaMA-3 with
substantially less calibration time.
On LLaMA-3-70B, ThinQuant reduces WikiText-2 perplexity
from $7.55$ to $5.63$ and C4 perplexity from $15.69$ to
$10.34$, while reducing total rotation-calibration time
from $111$ minutes to under $12$ minutes. Moreover, Table~\ref{tab:peak-disk} in the appendix shows
that ThinQuant requires several orders of magnitude less
disk space than DartQuant. DartQuant requires storing large
activation tensors, whereas ThinQuant selects and retains
only a tiny subset of activations for rotation calibration.
For LLaMA-3-70B, this reduces the required disk space
from $960$\,GiB to $0.51$\,GiB.
Finally, Table~\ref{tab:405b-scalability} shows that ThinQuant scales to
LLaMA-3.1-405B, completing rotation calibration in just over 2 hours on a single H200 GPU. Under W4A4 quantization,
replacing QuaRot's fixed Hadamard rotations with ThinQuant
reduces WikiText-2 perplexity from $5.82$ to $3.10$ with GPTQ
and from $3.48$ to $2.97$ with GPTAQ, using the same
asymmetric quantization setting as \citet{li2025gptaqefficientfinetuningfreequantization}. We provide additional experiments on W4A4 including a comparison to OmniQuant in Table \ref{tab:results-w4a4}.

\section{Conclusions and Limitations}

ThinQuant combines geometric activation selection with an exact
subspace reduction to make rotation learning substantially cheaper.
ThinQuant
achieves competitive low-bit quantization performance compared to state-of-the-art methods with much lower
calibration runtime, GPU memory, and disk storage, making
rotation calibration of LLaMA-3.1-405B possible on a single H200 GPU. 

\textbf{Limitations.} We remark that we optimize rotations separately prior to using weight quantization algorithms like GPTQ, though the two could in principle be solved jointly. An interesting research direction is to design dedicated methods that can efficiently optimize rotation and weight quantization simultaneously. Moreover, some recent works have shown that learning rotation and diagonal scaling transforms jointly can improve quantization, though this can require expensive calibration \citep{hu2025ostquantrefininglargelanguage}. We leave scalable and efficient rotation and scaling optimization to future work.

\bibliography{references}
\bibliographystyle{plainnat}

\clearpage
\appendix

\section{Proof of Theorem~\ref{thm:haar-selection}}
\label{sec:thm_proof}

\begin{proof}
Throughout the proof, the activations and the subset
$\mathcal I$ are fixed. Write $q=q(\mathcal I)$ and
$H=\max_{i\in[N]}\|x_i\|_2$. Recall that:
\[
K=\conv\{\pm x_i:i\in[N]\},
\qquad
K_{\mathcal I}=\conv\{\pm x_i:i\in\mathcal I\}.
\]
For $u\in\mathbb R^d$, define:
\[
f(u)
=
h_K(u)-h_{K_{\mathcal I}}(u)
=
\max_{i\in[N]}|u^{\mathsf T}x_i|
-
\max_{i\in\mathcal I}|u^{\mathsf T}x_i|
\]
Since $\mathcal I\subseteq[N]$, we have $f(u)\ge0$.
The absolute values also imply that
$f(-u)=f(u)$ and $f(tu)=t f(u)$ for every $t\ge0$.
We also have
$0\le f(u)\le h_K(u)\le H\|u\|_2$,
so $f(g)$ is square-integrable for
$g\sim\mathcal N(0,I_d)$.

We first show that the expected gap in objectives
is bounded by:
\begin{equation}
\mathbb E_R\!\left[
\left(F(R)-F_{\mathcal I}(R)\right)^2
\right]
\le
\mathbb E_g[f(g)^2]
\label{eq:gaussian-haar-gap}
\end{equation}
Write the rows of $R$ as
$r_1^{\mathsf T},\ldots,r_d^{\mathsf T}$.
For every rotation,
\begin{align*}
F(R)
&=
\max_{1\le k\le d}
\left[h_{K_{\mathcal I}}(r_k)+f(r_k)\right]\\
&\le
\max_{1\le k\le d}h_{K_{\mathcal I}}(r_k)
+
\max_{1\le k\le d}f(r_k)\\
&=
F_{\mathcal I}(R)+\max_{1\le k\le d}f(r_k).
\end{align*}
Since $F_{\mathcal I}(R)\le F(R)$ and $f\ge0$, this gives:
\begin{equation}
0
\le
\left(F(R)-F_{\mathcal I}(R)\right)^2
\le
\max_{1\le k\le d}f(r_k)^2
\le
\sum_{k=1}^d f(r_k)^2
\label{eq:rowwise-gap-bound}
\end{equation}

Since $d\ge2$, if $R\sim\operatorname{Haar}(\SO(d))$,
each row $r_k$ is marginally uniform on the unit sphere
$\mathbb S^{d-1}$
\citep{Meckes_2019}.
Let $v$ be uniform on $\mathbb S^{d-1}$.
Taking expectations in \eqref{eq:rowwise-gap-bound} therefore
gives:
\begin{equation}
\mathbb E_R\!\left[
\left(F(R)-F_{\mathcal I}(R)\right)^2
\right]
\le
d\,\mathbb E_v[f(v)^2].
\label{eq:haar-sphere-gap}
\end{equation}

Now take $g\sim\mathcal N(0,I_d)$ and write
$\rho=\|g\|_2$ and $v=g/\|g\|_2$.
Then $v$ is uniform on the sphere and independent of $\rho$
\citep{vershynin2018high}.
Moreover,
\[
\mathbb E[\rho^2]
=
\mathbb E\|g\|_2^2
=
\sum_{j=1}^d\mathbb E[g_j^2]
=
d
\]
Since $f(tu)=t f(u)$ for every $t\ge0$ and $\rho$ and $v$
are independent:
\[
\mathbb E_g[f(g)^2]
=
\mathbb E_{\rho,v}[\rho^2 f(v)^2]
=
\mathbb E[\rho^2]\mathbb E_v[f(v)^2]
=
d\,\mathbb E_v[f(v)^2]
\]
Combining this with \eqref{eq:haar-sphere-gap}
proves \eqref{eq:gaussian-haar-gap}.

It remains to bound the
right-hand side of \eqref{eq:gaussian-haar-gap}.
We do this by establishing two variance bounds:
\[
(1-q)\mathbb E_g[f(g)^2]
\le
\operatorname{Var}_g(f(g))
\le
2H^2q.
\]
For $q<1$, the lower bound will allow us to bound the
second moment by $\operatorname{Var}_g(f(g))/(1-q)$.
The upper bound will then give an estimate
in terms of $H$ and $q$.

We first establish the lower bound by relating $q$ to the
probability that $f(g)$ is nonzero.
For distinct indices $i,j$, a tie
$|g^{\mathsf T}x_i|=|g^{\mathsf T}x_j|$
can occur only if:
\[
g^{\mathsf T}(x_i-x_j)=0
\qquad\text{or}\qquad
g^{\mathsf T}(x_i+x_j)=0.
\]
Both vectors $x_i-x_j$ and $x_i+x_j$ are nonzero by
the identification of activations up to sign.
Each equality therefore describes a proper hyperplane,
which has probability zero under the standard Gaussian.
There are only finitely many pairs of indices, so the
maximizing index is unique almost surely.
Whenever this unique maximizing index belongs to
$\mathcal I$, the two support values coincide and $f(g)=0$.
Whenever it does not belong to $\mathcal I$, every selected
activation has a strictly smaller absolute projection,
and hence $f(g)>0$.
It follows that:
\begin{equation}
\operatorname{Pr}[f(g)>0]
=
\operatorname{Pr}\left[i^{\star}(g)\notin\mathcal I\right]
=
q(\mathcal I).
\label{eq:gap-positive-probability}
\end{equation}

The function $f(g)$ is zero outside the event
$\{f(g)>0\}$, which has probability $q$ by
\eqref{eq:gap-positive-probability}.
Cauchy--Schwarz therefore gives:
\begin{align}
\left(\mathbb E_g[f(g)]\right)^2
&=
\left(
\mathbb E_g\!\left[
f(g)\mathbf 1_{\{f(g)>0\}}
\right]
\right)^2
\notag\\
&\le
\mathbb E_g[f(g)^2]\,
\mathbb E_g\!\left[
\mathbf 1_{\{f(g)>0\}}^2
\right]
\notag\\
&=
q\,\mathbb E_g[f(g)^2]
\label{eq:gaussian-gap-mean-square}
\end{align}
Using the definition of variance and
\eqref{eq:gaussian-gap-mean-square}, we obtain the required
lower bound:
\begin{align}
\operatorname{Var}_g(f(g))
&=
\mathbb E_g[f(g)^2]
-
\left(\mathbb E_g[f(g)]\right)^2
\notag\\
&\ge
(1-q)\mathbb E_g[f(g)^2].
\label{eq:gaussian-gap-variance-lower}
\end{align}

We now establish the upper bound on the variance.
We use the strengthened Gaussian Poincar\'e inequality
for functions orthogonal to constants and linear functions
\citep[]{cordero2004bconjecture}.
Evenness ensures orthogonality to linear functions, but
does not imply a zero mean. To satisfy both conditions,
we apply the inequality to the function:
\[
\widetilde f(u)=f(u)-\mathbb E_g[f(g)]
\]
By construction, $\mathbb E_g[\widetilde f(g)]=0$.
Moreover, subtracting a constant preserves evenness, so
symmetry gives:
\[
\mathbb E_g[g_j\widetilde f(g)]=0,
\qquad j=1,\ldots,d.
\]
Thus, centering removes the constant component, while
evenness eliminates the linear component. Subtracting a constant also leaves the gradient unchanged,
$\nabla\widetilde f=\nabla f$ almost everywhere.
The strengthened inequality therefore gives:
\begin{equation}
\begin{aligned}
\operatorname{Var}_g(f(g))
&=
\mathbb E_g[\widetilde f(g)^2]\\
&\le
\frac12\mathbb E_g\|\nabla\widetilde f(g)\|_2^2
=
\frac12\mathbb E_g\|\nabla f(g)\|_2^2.
\end{aligned}
\label{eq:even-gaussian-poincare}
\end{equation}

To bound the right-hand side of
\eqref{eq:even-gaussian-poincare}, we next control the
gradient of $f$. We first show that $f$ is $2H$-Lipschitz,
which implies $\|\nabla f(u)\|_2\le2H$ wherever $f$ is
differentiable. Since $f=h_K-h_{K_{\mathcal I}}$, it suffices
to show that each support function is $H$-Lipschitz.
Indeed, for any $u,v\in\mathbb R^d$,
\begin{align*}
|h_K(u)-h_K(v)|
&\le
\max_{i\in[N]}
\left|
|u^{\mathsf T}x_i|-|v^{\mathsf T}x_i|
\right|\\
&\le
\max_{i\in[N]}|(u-v)^{\mathsf T}x_i|\\
&\le
H\|u-v\|_2.
\end{align*}
The same argument applies to $h_{K_{\mathcal I}}$.
By the triangle inequality, their difference therefore satisfies:
\begin{equation}
|f(u)-f(v)|
\le
2H\|u-v\|_2
\label{eq:gap-lipschitz}
\end{equation}

Both support functions are maxima of finitely many linear
functions. Their difference $f$ is therefore piecewise
linear and differentiable almost everywhere.
At any differentiability point $u$ satisfying $f(u)=0$,
the nonnegativity of $f$ implies that $u$ is a global
minimum. Hence $\nabla f(u)=0$.
At every other differentiability point,
\eqref{eq:gap-lipschitz} implies
$\|\nabla f(u)\|_2\le2H$.
Thus,
\[
\|\nabla f(u)\|_2^2
\le
4H^2\mathbf 1_{\{f(u)>0\}}
\]
Since the Gaussian measure is absolutely continuous,
\eqref{eq:gap-positive-probability} gives:
\begin{equation}
\mathbb E_g\|\nabla f(g)\|_2^2
\le
4H^2q
\label{eq:gaussian-gradient-bound}
\end{equation}
Applying \eqref{eq:even-gaussian-poincare} and then
\eqref{eq:gaussian-gradient-bound}, we obtain the required
upper bound:
\begin{equation}
\operatorname{Var}_g(f(g))
\le
2H^2q
\label{eq:gaussian-gap-variance}
\end{equation}

We now use these two variance bounds to control the
second moment required by \eqref{eq:gaussian-haar-gap}.
Suppose first that $q<1$.
Rearranging the lower bound
\eqref{eq:gaussian-gap-variance-lower} gives the first
inequality below, and substituting the upper bound
\eqref{eq:gaussian-gap-variance} gives the second:
\begin{equation}
\mathbb E_g[f(g)^2]
\le
\frac{\operatorname{Var}_g(f(g))}{1-q}
\le
\frac{2H^2q}{1-q}
\label{eq:gaussian-gap-second-moment}
\end{equation}
Substituting \eqref{eq:gaussian-gap-second-moment} into
\eqref{eq:gaussian-haar-gap} therefore yields:
\begin{equation}
\mathbb E_R\!\left[
\left(F(R)-F_{\mathcal I}(R)\right)^2
\right]
\le
\mathbb E_g[f(g)^2]
\le
\frac{2H^2q}{1-q}
\label{eq:rotation-gap-second-moment}
\end{equation}
Cauchy--Schwarz, followed by
\eqref{eq:rotation-gap-second-moment}, then gives:
\begin{equation}
\mathbb E_R[F(R)-F_{\mathcal I}(R)]
\le
\left(
\mathbb E_R\!\left[
\left(F(R)-F_{\mathcal I}(R)\right)^2
\right]
\right)^{1/2}
\le
H\sqrt{\frac{2q}{1-q}}
\label{eq:expected-gap-intermediate}
\end{equation}

For every $R\in\SO(d)$ and every activation,
\[
\|Rx_i\|_\infty
\le
\|Rx_i\|_2
=
\|x_i\|_2
\le
H
\]
Thus:
\begin{equation}
0\le F(R)-F_{\mathcal I}(R)\le H
\label{eq:trivial-gap-bound}
\end{equation}
If $0\le q\le1/3$, then $1-q\ge2/3$, and hence:
\[
\frac{2q}{1-q}\le3q.
\]
Equation~\eqref{eq:expected-gap-intermediate} therefore gives:
\[
\mathbb E_R[F(R)-F_{\mathcal I}(R)]
\le
H\sqrt{3q}.
\]
If $1/3\le q\le1$, then $\sqrt{3q}\ge1$, so
\eqref{eq:trivial-gap-bound} gives:
\[
\mathbb E_R[F(R)-F_{\mathcal I}(R)]
\le
H
\le
H\sqrt{3q}.
\]
Substituting $q=q(\mathcal I)$ proves
\eqref{eq:expected-selection-bound}.
\end{proof}

\section{Reduced ADMM updates}
\label{app:updates}

Let $C=C_d$ and let $\Lambda,\Gamma$ be scaled dual variables
for $VS-W=0$ and $CW-Z=0$. With $V\in\St(d,r)$, we write the
augmented Lagrangian of \eqref{eq:admm-split} as:
\begin{equation}
\begin{aligned}
    \mathcal{L}_{\rho}
    ={}& \|Z\|_4^4
    +\frac{\rho}{2}\|VS-W+\Lambda\|_F^2
    +\frac{\rho}{2}\|CW-Z+\Gamma\|_F^2
    -\frac{\rho}{2}
    \bigl(\|\Lambda\|_F^2+\|\Gamma\|_F^2\bigr)
\end{aligned}
\label{eq:augmented-lagrangian}
\end{equation}
Holding $\rho$ fixed within each sweep, we update
\begin{align}
    Z^+&=\operatorname{prox}_{\|\cdot\|_4^4/\rho}(CW+\Gamma),
    \label{eq:prox}\\
    W^+&=(I_d-\tfrac12 C)
    \bigl(VS+\Lambda+C(Z^+-\Gamma)\bigr),
    \label{eq:W}\\
    V^+&=PQ^{\mathsf T},
    \quad P\Sigma Q^{\mathsf T}
    =\operatorname{svd}_{\mathrm{thin}}
    \bigl((W^+-\Lambda)S^{\mathsf T}\bigr),
    \label{eq:V}\\
    \Lambda^+&=\Lambda+V^+S-W^+,
    \qquad \Gamma^+=\Gamma+CW^+-Z^+.
    \label{eq:dual}
\end{align}
For an input entry $a$, the proximal value is the unique real
solution of $4z^3+\rho(z-a)=0$.
The $W$-update uses $(I_d+C)^{-1}=I_d-C/2$,
while the $V$-update is an orthogonal Procrustes problem because
$\|VS\|_F^2=\|S\|_F^2$ on $\St(d,r)$.

\paragraph{Penalty selection.}
For $R_1$, we initialize
$\rho_0=10\|CU\|_F^2/(dB)$, ten times the mean squared entry
of the centered activation matrix.
We then adapt $\rho$ every $20$ iterations by primal--dual residual
balancing \citep{boyd2011admm}.
Let:
\begin{align*}
    r_{\mathrm{p}}
    &=\max\bigl(\|W^+-V^+S\|_F,\ \|CW^+-Z^+\|_F\bigr),\\
    r_{\mathrm{d}}
    &=\rho\max\bigl(\|W^+-W\|_F,\ \|Z^+-Z\|_F\bigr).
\end{align*}
At each adjustment step, set:
\[
    \rho^+=
    \begin{cases}
        1.25\,\rho,
        & r_{\mathrm{p}}>10\,r_{\mathrm{d}},\\
        \rho/1.25,
        & r_{\mathrm{d}}>10\,r_{\mathrm{p}},\\
        \rho,
        & \text{otherwise},
    \end{cases}
\]
and clip the result to $[10^{-8},\max(10\rho_0,10^{-3})]$. The local $R_2$ problems are generally well behaved in our
experiments, so each $R_2$ solve uses a fixed $\rho=10$.

\paragraph{Iteration counts and stopping.}
Each solve alternates one $R_1$ problem with $R_{2, \ell}$ problem per layer.
$R_1$ runs for at most $1000$ iterations and for at least $300$.
After iteration $300$ it stops when both the primal and dual residuals are below $10^{-4}$:
\[
    \max\bigl(\|W-VS\|_F,\ \|CW-Z\|_F\bigr)<10^{-4},
    \qquad
    \rho\max\bigl(\|W^+-W\|_F,\ \|Z^+-Z\|_F\bigr)<10^{-4}.
\]
Each $R_2$ solve runs for $300$ iterations.
In every case the returned factor is the feasible iterate with the smallest objective value.

\section{Ablations and Additional Experiments}

\paragraph{Activation selection.}
Table~\ref{tab:selection-ablation} compares four selection rules
on LLaMA-3-70B under W4A4KV4 quantization, keeping the activation
budget and subsequent rotation optimization fixed.
Algorithm~\ref{alg:directional-probes} achieves a WikiText-2
perplexity of $5.63$, compared with $6.94$ when selecting the
largest $\ell_\infty$ activations after the initial Hadamard
rotation and $7.00$ when selecting them in the original basis.
Uniformly sampling the same number of activations performs
substantially worse, with a perplexity of $29.51$.
These results show that the choice of calibration activations
is important at such a small budget.

\paragraph{Global and local rotations.}
The optimization ablation in Table~\ref{tab:selection-ablation}
relates to the contribution of learning the global residual
rotation $R_1$ and the local rotations $R_2$.
Optimizing only $R_1$, while leaving $R_2$ at initialization,
achieves a perplexity of $5.69$, close to the $5.63$ obtained
when both are optimized.
Optimizing only $R_2$, while keeping $R_1$ at its initial
Hadamard rotation, gives a perplexity of $6.01$.
Thus, learning the global residual basis is sufficient to
approach the performance of the full method in this setting,
while also optimizing the local rotations provides a further
improvement.

\begin{table}[t]
  \centering
  \small
  \setlength{\tabcolsep}{4pt}
  \renewcommand{\arraystretch}{1.08}
  \begin{tabular}{@{}lrrrrrrr@{}}
    \toprule
    & \multicolumn{4}{c}{\textbf{Selection}}
    & \multicolumn{3}{c}{\textbf{Optimization}}
    \\
    \cmidrule(lr){2-5}
    \cmidrule(lr){6-8}
    & \shortstack{Algorithm \ref{alg:directional-probes}}
    & \shortstack{$\operatorname{Top}_B$ $(\|R_0 x\|_\infty)$}
    & \shortstack{$\operatorname{Top}_B$ $(\|x\|_\infty)$}
    & \shortstack{Random $B$}
    & $R_1,R_2$
    & \shortstack{$R_1$ only}
    & \shortstack{$R_2$ only}
    \\
    \midrule
    \textbf{Wiki-2 ppl} $\downarrow$
    & \cellcolor{darkgreen!10}\textbf{5.63}
    & 6.94
    & 7.00
    & 29.51
    & \cellcolor{darkgreen!10}\textbf{5.63}
    & 5.69
    & 6.01
    \\
    \bottomrule
  \end{tabular}
  \caption{
    LLaMA-3 70B W4A4KV4 ablations.
    \textbf{Selection}: all four columns optimize the same ThinQuant rotation and differ only in how the $B$ tokens are chosen.
    The first column is the Gaussian direction procedure of Algorithm~\ref{alg:directional-probes}.
    $\operatorname{Top}_B(\|R_0 x\|_\infty)$ scores the tokens after the initial Hadamard $R_0$.
    $\operatorname{Top}_B(\|x\|_\infty)$ keeps the tokens of largest residual $\ell_\infty$ norm in the original basis.
    Random samples the same budget uniformly.
    \textbf{Optimization}: we compare keeping only the global rotation ($R_1$) or only the local rotatons ($R_2$). 
    $R_1,R_2$ is the full Llama-3-70B run.
    $R_1$ only optimizes the residual rotation and leaves $R_2$ at initialization.
    $R_2$ only optimizes this rotation and leaves $R_1$ at the initial Hadamard.
    Best results in each block are \textbf{bold}.
  }
  \label{tab:selection-ablation}
\end{table}

\paragraph{Calibration cost breakdown.}
Table~\ref{tab:peak-disk} separates rotation optimization from
the total calibration procedure and reports the associated
disk storage.
On LLaMA-3-70B, ThinQuant reduces optimization time from
$74$ minutes $46$ seconds to $1$ minute $22$ seconds,
approximately a $55\times$ speedup over DartQuant.
The total calibration time decreases from $111$ minutes to 11:57 minutes, approximately a $10\times$ speedup.
The smaller reduction in total time reflects the cost of
activation collection, selection, and data preparation, which
together account for most of ThinQuant's calibration time.
Across all five models, ThinQuant also reduces disk storage
by more than three orders of magnitude.
For example, LLaMA-3-70B requires $0.51$\,GiB rather than
$960$\,GiB.
These savings arise from not needing to store the large
calibration set used by DartQuant.

\begin{table}[t]
  \centering
  \small
  \setlength{\tabcolsep}{3.5pt}
  \renewcommand{\arraystretch}{1.08}

  \begin{tabular}{@{}llrrrrr@{}}
    \toprule
    \textbf{Cost}
    & \textbf{Method}
    & \shortstack{\textbf{LLaMA-3}\\\textbf{8B}}
    & \shortstack{\textbf{LLaMA-3}\\\textbf{70B}}
    & \shortstack{\textbf{LLaMA-2}\\\textbf{7B}}
    & \shortstack{\textbf{LLaMA-2}\\\textbf{13B}}
    & \shortstack{\textbf{LLaMA-2}\\\textbf{70B}}
    \\
    \midrule

    \rowcolor{darkgreen!10}
    \cellcolor{white}\multirow{2}{*}{Optimization}
    & \texttt{ThinQuant}
    & \textbf{23.4s}
    & \textbf{1:22}
    & \textbf{22.0s}
    & \textbf{34.5s}
    & \textbf{1:31}
    \\
    & \texttt{DartQuant}
    & 12:14
    & 74:46
    & 12:05
    & 21:22
    & 60:47
    \\
    \midrule

    \rowcolor{darkgreen!10}
    \cellcolor{white}\multirow{2}{*}{Total calibration}
    & \texttt{ThinQuant}
    & \textbf{1:56}
    & \textbf{11:57}
    & \textbf{1:43}
    & \textbf{2:51}
    & \textbf{11:42}
    \\
    & \texttt{DartQuant}
    & 19:03
    & 111:00
    & 20:25
    & 31:42
    & 99:57
    \\
    \midrule

    \rowcolor{darkgreen!10}
    \cellcolor{white}\multirow{2}{*}{Disk storage}
    & \texttt{ThinQuant}
    & \textbf{0.13}
    & \textbf{0.51}
    & \textbf{0.13}
    & \textbf{0.20}
    & \textbf{0.51}
    \\
    & \texttt{DartQuant}
    & 192
    & 960
    & 193
    & 300
    & 961
    \\
    \bottomrule
  \end{tabular}

  \caption{
    Breakdown of rotation-calibration cost.
    Optimization measures rotation optimization only, while total
    calibration additionally includes activation collection, selection,
    training-data preparation, and folding the rotation into the weights.
    Times are in \texttt{mm:ss} unless marked in seconds.
    Disk storage is in GiB.
    Lower values are better, and best results are \textbf{bold}.
  }
  \label{tab:peak-disk}
\end{table}

\begin{table}[t]
  \centering
  \small
  \setlength{\tabcolsep}{3.5pt}
  \renewcommand{\arraystretch}{1.08}

  \begin{tabular}{@{}llrrr@{\hspace{1em}}rr@{}}
    \toprule
    & &
    \multicolumn{3}{c}{\textbf{Quantization performance}}
    &
    \multicolumn{2}{c}{\textbf{Calibration cost}}
    \\
    \cmidrule(lr){3-5}
    \cmidrule(lr){6-7}
    \textbf{Model}
    & \textbf{Method}
    & \textbf{Wiki-2} $\downarrow$
    & \textbf{C4} $\downarrow$
    & \textbf{0-shot} $\uparrow$
    & \textbf{Runtime} $\downarrow$
    & \textbf{Peak GPU} $\downarrow$
    \\
    \midrule

    \rowcolor{darkgreen!10}
    \cellcolor{white}\multirow{3}{*}{\textbf{LLaMA-3 8B}}
    & \texttt{ThinQuant}
    & \textbf{7.75} & \textbf{12.55} & \textbf{60.87}
    & \textbf{1:56} & 18.5
    \\
    & \texttt{DartQuant}
    & 8.03 & 13.09 & 59.98
    & 19:03 & 37.1
    \\
    & \texttt{OmniQuant}
    & 71.32 & 92.86 & 32.00
    & 54:57 & \textbf{17.2}
    \\
    \midrule

    \rowcolor{darkgreen!10}
    \cellcolor{white}\multirow{3}{*}{\textbf{LLaMA-3 70B}}
    & \texttt{ThinQuant}
    & \textbf{5.52} & \textbf{10.14} & \textbf{66.18}
    & \textbf{11:57} & \textbf{31.9}
    \\
    & \texttt{DartQuant}
    & 7.15 & 14.88 & 59.05
    & 111:00 & 134.9
    \\
    & \texttt{OmniQuant}
    & 482.73 & 775.47 & --
    & 406:49 & 127.0
    \\
    \midrule

    \rowcolor{darkgreen!10}
    \cellcolor{white}\multirow{3}{*}{\textbf{LLaMA-2 7B}}
    & \texttt{ThinQuant}
    & \textbf{6.12} & \textbf{8.32} & \textbf{57.58}
    & \textbf{1:43} & 18.2
    \\
    & \texttt{DartQuant}
    & 6.25 & 8.52 & 56.93
    & 20:25 & 75.1
    \\
    & \texttt{OmniQuant}
    & 11.21 & 16.05 & 45.68
    & 53:59 & \textbf{16.2}
    \\
    \midrule

    \rowcolor{darkgreen!10}
    \cellcolor{white}\multirow{3}{*}{\textbf{LLaMA-2 13B}}
    & \texttt{ThinQuant}
    & \textbf{5.43} & \textbf{7.59} & 61.25
    & \textbf{2:51} & \textbf{21.9}
    \\
    & \texttt{DartQuant}
    & 5.46 & 7.63 & \textbf{61.65}
    & 31:42 & 51.1
    \\
    & \texttt{OmniQuant}
    & 10.64 & 16.62 & 44.49
    & 93:07 & 27.3
    \\
    \midrule

    \rowcolor{darkgreen!10}
    \cellcolor{white}\multirow{3}{*}{\textbf{LLaMA-2 70B}}
    & \texttt{ThinQuant}
    & \textbf{3.88} & 6.20 & 67.01
    & \textbf{11:42} & \textbf{33.2}
    \\
    & \texttt{DartQuant}
    & 3.89 & \textbf{6.19} & \textbf{67.31}
    & 99:57 & 73.1
    \\
    & \texttt{OmniQuant}
    & 8.43 & 12.19 & 49.76
    & 404:08 & 125.9
    \\
    \bottomrule
  \end{tabular}

  \caption{
    W4A4 quantization performance and calibration cost.
    Weights and activations are 4-bit; the KV cache is left at 16 bits.
    The 0-shot column averages nine downstream tasks.
    For ThinQuant and DartQuant, runtime is total rotation-calibration time,
    including activation collection, selection, and training-data preparation.
    Both reuse the rotations from the W4A4KV4 runs, so those calibration costs
    match Table~\ref{tab:results}.
    OmniQuant runtime is activation-scale collection plus LET/LWC fitting.
    GPU denotes peak GPU process memory during calibration, in GiB.
    Times are in \texttt{mm:ss}.
    Best results are \textbf{bold}.}
  \label{tab:results-w4a4}
\end{table}

\end{document}